\documentclass[11pt]{article}
\usepackage[utf8]{inputenc}
\usepackage[T1]{fontenc}
\usepackage[margin=1in]{geometry}
\usepackage{amsmath,amssymb,amsthm,mathtools}
\usepackage{graphicx}
\usepackage{booktabs,array}
\usepackage{xcolor}
\usepackage{enumitem}
\usepackage[numbers,sort&compress]{natbib}
\usepackage{tikz}
\usetikzlibrary{arrows.meta,positioning}
\usepackage[hidelinks]{hyperref}
\usepackage[capitalize,noabbrev]{cleveref}

\crefformat{equation}{(#2#1#3)}
\Crefformat{equation}{Equation~(#2#1#3)}
\newtheorem{proposition}{Proposition}
\newtheorem{lemma}[proposition]{Lemma}
\newtheorem{corollary}[proposition]{Corollary}

\setlist{itemsep=2pt,topsep=4pt}

\newcommand{\sgn}{\operatorname{sgn}}
\newcommand{\Prob}{\mathrm{P}}
\newcommand{\E}{\mathbb{E}}
\newcommand{\N}{\mathcal{N}}
\newcommand{\R}{\mathbb{R}}
\newcommand{\data}{\mathcal{D}}
\DeclareMathOperator*{\argmax}{arg\,max}
\newcommand{\unseen}{\textup{\textsc{unseen}}}
\newcommand{\jammed}{\textup{\textsc{jammed}}}
\newcommand{\free}{\textup{\textsc{free}}}
\newcommand{\vref}{v_{\mathrm{ref}}}
\IfFileExists{open_world/short/l2a_short.tex}{\newcommand{\sroot}{open_world/short/}}{\newcommand{\sroot}{}}

\title{Open-World Law-to-Action:\\ Vision--Language Priors and Touch for Parts with Hidden Physics}
\author{Mohamed A.~Mabrok\\ Qatar University\\ \texttt{m.a.mabrok@gmail.com}}
\date{}

\begin{document}
\maketitle

\begin{abstract}
A robot in a home it has never seen must open doors, drawers, lids and taps whose physics a picture does not show.
Instructions such as ``open the drawer halfway, gently'' depend on that physics: how far a part travels, how it
resists, and what locks it. We present Open-World Law-to-Action (L2A). A vision--language model (VLM) proposes, for
each movable part, its motions, the terms of its law of motion, and what may lock it. Each proposal becomes a prior
with a floor, so that touch can always overturn it. The robot identifies each part's law online by Bayesian linear
regression with closed-form evidence, for parts that move in one way or in two, and keeps a hypothesis that an unseen
part or motion locks the target. We prove that touch alone cannot rule this hypothesis out, derive when another look
is worth its cost, and show that a constant load and friction cannot be told apart while a part moves one way. All
experiments are in simulation, developed on some scenes and run once on held-out ones. In new kitchens the robot
completes 80\% of tasks from the VLM's proposals, against 98\% when every part is known. Proposed physics raises
success on objects that need new terms from 353 to 378--383 of 400 episodes, and a second look raises success in
mixed scenes from 284 to 329 of 366 tasks. A mobile robot that takes nothing from the simulator completes 70 of 100
requests in new scenes, against 62 when the simulator says which part each box is (after one change made between two
held-out tests).
\end{abstract}

%=====================================================================
\section{Introduction}
\label{sec:intro}

Ask a robot to ``open the cupboard halfway, gently''. To do this well it needs three facts that a
picture does not give. It needs the \emph{range}: ``halfway'' is half of the travel up to a stop that
may be hidden inside the cupboard. It needs the \emph{resistance}: friction, a closing spring or the
weight of a lid decide how much force moves the door, and too much force makes the motion fast and not
gentle. It needs the \emph{locks}: a catch may hold the door shut until a knob is turned or a handle is
pressed, and the knob may not be in view. People learn these facts by trying. They start from a guess
based on what they see, pull a little, and feel what happens.

Current methods use only one of these two sources well. Vision--language--action
policies~\citep{brohan2023rt2,kim2024openvla,black2025pi0} map images and words to motor commands. They
learn hidden physics only implicitly, from demonstrations, and a policy that imitates a teacher who
already knows the stop never learns to gather that information itself~\citep{weihs2021imitationgap}.
Planners built on language and vision models~\citep{ahn2022saycan,huang2023voxposer,huang2024rekep}
reason about what the scene shows, but a hidden stop or catch is not in the image. Interactive
perception and system identification~\citep{bohg2017interactive,sturm2011articulation,hausman2015active,martinmartin2014online}
learn how parts move by acting on them, but they assume a known model class and a known set of parts.
Recent systems ask a VLM for physical parameters and then refine them by
interaction~\citep{wang2026phys2real,cicek2026coral,gao2026joda}. What is missing is a principled way to
use a VLM's claims about the \emph{structure} of a law and about \emph{which parts and motions exist and
what locks what}, which stays correct when a claim is wrong and when a part was never seen.

This paper gives such a method (\cref{fig:overview}). The VLM's claims only set the robot's starting
beliefs, and measurements change them. Every law of motion, for a part with one coordinate or two, is
built from a library of standard terms, each linear in its unknown weight. Learning a law is then Bayesian
linear regression, with a closed-form posterior and a closed-form evidence for each candidate structure.
Each prior has a floor, so that no claim can make a hypothesis impossible; only measurements can. The
robot keeps one more hypothesis than a closed-world agent: that the part is locked by a part or a motion
nobody has seen. Instructions are compiled into functions of the law. Uncertainty about the law thus
becomes uncertainty about success, and the robot experiments, looks again, acts or refuses by comparing
success probabilities.

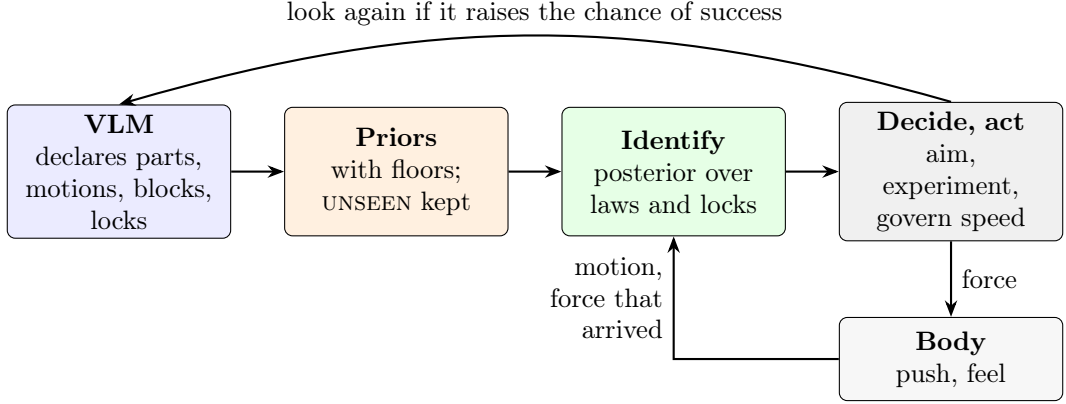
\begin{figure}[tbp]
\centering
\begin{tikzpicture}[node distance=6mm and 7mm, every node/.style={font=\small},
  box/.style={draw, rounded corners=3pt, align=center, minimum height=17mm, text width=27mm, fill=#1},
  box/.default=white, arr/.style={-{Stealth[length=2.2mm]}, thick}]
\node[box=blue!8] (vlm) {\textbf{VLM}\\declares parts,\\motions, blocks,\\locks};
\node[box=orange!12, right=of vlm] (prior) {\textbf{Priors}\\with floors;\\\unseen{} kept};
\node[box=green!10, right=of prior] (id) {\textbf{Identify}\\posterior over\\laws and locks};
\node[box=gray!12, right=of id] (act) {\textbf{Decide, act}\\aim,\\experiment,\\govern speed};
\node[box=gray!6, minimum height=11mm, below=10mm of act] (body) {\textbf{Body}\\push, feel};
\draw[arr] (vlm) -- (prior);
\draw[arr] (prior) -- (id);
\draw[arr] (id) -- (act);
\draw[arr] (act) -- node[right]{force} (body);
\draw[arr] (body.west) -| node[pos=0.75, left, align=right]{motion,\\force that\\arrived} (id.south);
\draw[arr] (act.north) to[bend right=16] node[above]{look again if it raises the chance of success} (vlm.north);
\end{tikzpicture}
\caption{Open-World Law-to-Action. The VLM's declaration sets the priors. Touch updates the posterior over
each part's law and locks. The agent looks again only when a look raises the probability of success by
more than it costs.}
\label{fig:overview}
\end{figure}

\paragraph{Contributions.}
\begin{enumerate}[leftmargin=*]
\item \emph{One model for one-motion and two-motion parts} (\cref{sec:problem}). Laws with one or two
  coordinates share one library and one regression. A catch inside a part (a handle that must be pressed)
  and a lock between parts (a knob that frees a drawer) are the same kind of hypothesis.
\item \emph{Declarations as floored priors} (\cref{sec:priors}). We turn a VLM's claims into priors over
  law structure, range and locks, and show that the floor bounds the evidence needed to overturn any claim
  by a constant (\cref{prop:overturn,cor:steps}).
\item \emph{Unseen parts and motions} (\cref{sec:unseen}). We prove that touch cannot separate ``locked by
  something undeclared'' from ``jammed'' (\cref{lem:unseen}), so a robot in an open world cannot certify by
  touch alone that a task is impossible. We derive the value of another look, with and without latency.
\item \emph{An identifiability result with two uses} (\cref{sec:ident}). On motion in one direction a
  constant load and Coulomb friction share one regression column. We give the variance that a short
  reversal achieves, and show that ignoring the problem makes the feed-forward short by exactly twice the
  load (\cref{lem:alias,prop:reversal,cor:twice}). This explains and fixes failures on gas-strut flaps and
  on drawers that must be lifted.
\item \emph{A held-out evaluation, and a path off the simulator} (\cref{sec:ladder,sec:experiments}). Every
  experiment was developed on one set of scenes, frozen, and run once on new scenes. We state, for each
  experiment, what the robot still takes from the simulator, and remove these step by step, down to a
  robot that sees only through its cameras and feels only through its hand.
\end{enumerate}

%=====================================================================
\section{Problem}
\label{sec:problem}

\subsection{Parts, coordinates and laws}

A scene contains a set $\mathcal P$ of movable parts. The robot does not know the set. A part has one or
two \emph{task coordinates}, the quantities an instruction refers to: the angle of a door, the extension
of a drawer, the lift and the turn of a mixer tap. We write them as $q\in\R^n$, $n\in\{1,2\}$, with $q=0$
at the closed position. The robot applies a generalised force $u\in\R^n$ along them (a force for a slide,
a torque for a hinge) and measures $q$ and $\dot q$ with noise.

The \emph{law} of a part says how $q$ responds to $u$:
\begin{equation}
  M\ddot q \;=\; u \;-\; D\dot q \;-\; F\sgn(\dot q) \;-\; \sum_{j\in h}\alpha_j\,f_j(q,\dot q),
  \qquad 0\le q_i\le q_{s,i}.
  \label{eq:law}
\end{equation}
Here $M$ is the symmetric inertia matrix, $D$ and $F$ are diagonal matrices of viscous damping and Coulomb
friction (the sign is taken per coordinate), and $q_{s,i}$ is the end stop of coordinate $i$. The sum
collects the part's other physics. Each term is a known \emph{feature} $f_j:\R^n\times\R^n\to\R^n$ times an
unknown weight $\alpha_j$. A feature comes from a \emph{block} of a physics library $\mathcal B$
(\cref{tab:library}). The \emph{structure} $h\subseteq\mathcal B$ is the set of blocks present, and the law
of a part is the triple $(h,\psi,q_s)$ with continuous parameters $\psi=(M,D,F,\alpha)$. For one coordinate,
\cref{eq:law} reads $m\ddot q=u-b\dot q-F_c\sgn(\dot q)-\sum_j\alpha_jf_j(q,\dot q)$. We use this scalar form
in examples.

\begin{table}[tbp]
\centering
\small
\caption{The library of blocks. Each block adds a feature that is linear in its weight, so a new block adds
a column to a regression and nothing else. Parts with one coordinate use the upper blocks; parts with two
use springs, couplings and constant forces (16 structures).}
\label{tab:library}
\begin{tabular}{@{}l l l@{}}
\toprule
Block & Feature (weight) & Example \\
\midrule
Return spring on $q_i$ & $q_i e_i$ ($k_i$) & a self-closing door \\
Detents, period $p$ & $\sin(2\pi q/p)$ & clicks of a stove knob \\
Gravity load & $\cos q$, $\sin q$ & a lid that falls shut \\
Constant force on $q_i$ & $e_i$ ($c_i$, signed) & a gas strut; the weight on a lift \\
End damper & $\dot q\,w(q)$, $w\to1$ near the seat & a soft-close door \\
\midrule
Cross spring & $(q_2,\,q_1)$ ($k_{12}$) & a drawer that sags as it comes out \\
Cross inertia & off-diagonal $m_{12}$ of $M$ & a lever coupled through its mass \\
\midrule
End stop & event at $q_i=q_{s,i}$ & where a motion ends \\
Latch & event & a motion that moves only after another is past a level \\
\bottomrule
\end{tabular}
\end{table}

Friction makes the law hybrid. A coordinate at rest stays at rest while the force on it, less the
conservative forces, is within its friction band:
\begin{equation}
  \dot q_i=0 \ \text{ and }\ \Big|\,u_i-\sum_{j\in h}\alpha_j f_{j,i}(q,0)\Big|\le F_i
  \quad\Longrightarrow\quad \ddot q_i=0 .
  \label{eq:stick}
\end{equation}
The method relies on one property of \cref{eq:law}, not on any particular block: \emph{every unknown
coefficient multiplies a known function of the state}. This makes identification a linear regression
(\cref{sec:ident}) and lets the library grow without a new algorithm.

\subsection{Locks, inside a part and between parts}

Call each coordinate of each part a \emph{motion}, and let $\mathcal K$ be the set of motions the robot has
been told about. A motion $x$ may be \emph{latched}: it cannot leave its closed zone unless another motion
$y$ is past a level $\tau$. We call $y$ a \emph{releaser} of $x$. The releaser may be another motion of the
same part (a drawer slides only after it is lifted; a door opens only after its handle is pressed) or a
motion of another part (a knob frees a drawer). A motion may also be \emph{jammed}: nothing releases it.
For each motion $x$ the robot keeps a posterior over the \emph{lock hypotheses}
\begin{equation}
  \mathcal C_x \;=\; \{\free,\ \jammed,\ \unseen\}\ \cup\ \{(y,\tau):\ y\in\mathcal K,\ y\neq x,\ \tau\in\mathcal T\},
  \label{eq:lockset}
\end{equation}
where $(y,\tau)$ means ``$x$ is released once $q_y>\tau$'' and $\mathcal T$ is a grid of levels. The
hypothesis \unseen{} says that the releaser is a motion outside $\mathcal K$: a part the VLM did not
declare, or a motion it did not declare on a part it did. A closed-world agent does not have this
hypothesis.

\subsection{What the robot senses}

The robot has two senses with opposite strengths.
\begin{itemize}[leftmargin=*]
\item \emph{Sight, through a VLM.} A picture of the scene returns a \emph{declaration}: a list of parts,
  and for each part its motions (how many, of what kind, in which direction), a guess of each range, a
  weight for each block, and guesses of what may lock it. A declaration covers the whole scene at once. It
  also misses parts and motions, invents parts, and cannot see hidden stops, springs or catches.
\item \emph{Touch, through the body.} Pushing a part returns its motion under a measured force. Touch is
  slow and covers one part at a time, but it measures the law itself.
\end{itemize}
A declared part with nothing movable behind it is a \emph{phantom}. A part or motion the VLM missed does
not exist for the robot until a later \emph{look}, a new request to the VLM about another view, declares
it.

\subsection{Instructions as functions of the law}

An instruction is compiled into goals and constraints that are functions of the law (\cref{tab:words}).
``Halfway'' means $q\in[0.4\,q_s,\ 0.6\,q_s]$ for \emph{this} part's stop. ``One click'' means the next
stable minimum of \emph{this} knob's potential, which may not exist. ``Gently'' bounds the speed, and the
speed at which the part meets its stop. ``Hotter'' on a mixer tap is a goal on its turn coordinate, and
``more water'' on its lift. Because the goal is a function of the law, uncertainty about the law becomes
uncertainty about success. An episode succeeds if the goal holds at the end and no constraint is broken at
any time. If the instruction is impossible under the true law (the motion is jammed, or its only releaser is
forbidden), the episode succeeds if the robot says so. The objective is the probability of success within a
time budget. Time spent looking counts like time spent pushing.

\begin{table}[tbp]
\centering
\small
\caption{Instruction words as functions of a law $(h,\psi,q_s)$. $\mathcal M$ is the set of stable minima of
the identified potential above the start $q_0$.}
\label{tab:words}
\begin{tabular}{@{}l l l@{}}
\toprule
Words & Meaning under the law & Impossible when \\
\midrule
``fully'' & $q^\star=q_s$ & jammed, or only releaser forbidden \\
``halfway'' & $q^\star\in[0.4\,q_s,\ 0.6\,q_s]$ & as above \\
``one click'' & $q^\star=\min\{q\in\mathcal M:\ q>q_0\}$ & no stable detent above $q_0$ \\
``gently'' & $|\dot q_i|\le v_{\max}$; speed at a stop $\le v_{\mathrm{imp}}$ & never \\
``without touching X'' & $q_X(t)=q_X(0)$ for all $t$ & X is the only releaser \\
\bottomrule
\end{tabular}
\end{table}

As an example of a word that depends on the law, take a knob with a spring and detents. Its potential is
$V(q)=\tfrac12 kq^2-\tfrac{Ap}{2\pi}\cos(2\pi q/p)$, so $V'(q)=kq+A\sin(2\pi q/p)$. A click is a minimum of
$V$ away from zero. Since $V''(q)=k+\tfrac{2\pi A}{p}\cos(2\pi q/p)$, if $2\pi A/p\le k$ then $V''\ge0$
everywhere, $V$ is convex, and its only minimum is $q=0$. Clicks away from zero therefore need $2\pi A/p>k$.
A strong enough spring removes every click, even when the knob has detents. Whether ``one click'' is
possible is a fact about the law, not about how the knob looks.

%=====================================================================
\section{Identifying a law by touch}
\label{sec:ident}

This section shows how the robot learns a part's law from its own pushes, given a prior over structures.
\Cref{sec:priors} then shows where the prior comes from.

\subsection{Each structure is a linear regression}

Moving the known terms of \cref{eq:law} to one side gives $M\ddot q+D\dot q+F\sgn(\dot q)+\sum_{j\in h}\alpha_j
f_j(q,\dot q)=u$. Every coefficient (the entries of $M$, $D$, $F$ and the weights $\alpha$) multiplies a
measured quantity, so
\begin{equation}
  u \;=\; Y_h(q,\dot q,\ddot q)\,\theta_h ,
  \label{eq:linear}
\end{equation}
with a known matrix $Y_h$ of $n$ rows and an unknown coefficient vector $\theta_h$. We fit it in one of two
equivalent forms. For two coordinates we keep the force form: each time step gives two rows of
\cref{eq:linear}, with accelerations from finite differences, and 12 coefficients for the largest structure.
For one coordinate we divide by $m$, which puts the noise on the measured acceleration. Take one control
step of length $\Delta t$ with the force held constant and the velocity keeping its sign. Integrating
\cref{eq:law} over the step, with the integrand at the midpoint $\bar q=\tfrac12(q_t+q_{t+1})$,
$\bar v=\tfrac12(\dot q_t+\dot q_{t+1})$, gives $m(\dot q_{t+1}-\dot q_t)=\Delta t\big(u_t-b\bar v-F_c\sgn(\bar
v)-\sum_j\alpha_j f_j(\bar q,\bar v)\big)+O(\Delta t^3)$. Dividing by $m\,\Delta t$,
\begin{equation}
  y_t\;\coloneqq\;\frac{\dot q_{t+1}-\dot q_t}{\Delta t}\;=\;\phi_h^{\!\top}\beta_h+\varepsilon_t,\qquad
  \phi_h=\big[\,u_t,\ -\bar v,\ -\sgn(\bar v),\ -f_j(\bar q,\bar v)\,\big]^{\!\top},\quad
  \beta_h=\big[\,\tfrac1m,\ \tfrac bm,\ \tfrac{F_c}m,\ \tfrac{\alpha_j}m\,\big]^{\!\top}.
  \label{eq:regression}
\end{equation}
Both forms are a linear regression $y=\Phi_h\beta_h+\varepsilon$ with Gaussian noise
$\varepsilon\sim\N(0,\sigma^2I)$ of unknown variance, and everything below applies to either. We use only
\emph{clean} steps, in which the part slides in one direction, away from its seat and from any detected
stop. Steps where it sticks, touches a stop or is latched carry other information, which \cref{sec:events}
uses.

\subsection{Posterior and evidence in closed form}

For each structure we use the conjugate normal--inverse-gamma prior, $\sigma^2\sim\mathrm{IG}(a_0,b_0)$ and
$\beta_h\mid\sigma^2\sim\N(m_0,\sigma^2V_0)$. Stack $n$ rows into $\Phi_h\in\R^{n\times d_h}$ and $y\in\R^n$. The
posterior is again normal--inverse-gamma~\citep{bishop2006prml}, with
\begin{align}
  V_n^{-1}&=V_0^{-1}+\Phi_h^{\!\top}\Phi_h, & m_n&=V_n\big(V_0^{-1}m_0+\Phi_h^{\!\top}y\big),\label{eq:post1}\\
  a_n&=a_0+\tfrac n2, & b_n&=b_0+\tfrac12\big(y^{\!\top}y+m_0^{\!\top}V_0^{-1}m_0-m_n^{\!\top}V_n^{-1}m_n\big).\label{eq:post2}
\end{align}
The precision $V_n^{-1}$ grows with every informative row. The mean $m_n$ blends the prior guess with the
least-squares fit. The scale $b_n$ collects what the structure fails to explain. Integrating out $\beta_h$
and $\sigma^2$ (\cref{app:nig}) gives the \emph{evidence}, the probability of the data under the structure:
\begin{equation}
  p(y\mid h)\;=\;(2\pi)^{-n/2}\,\underbrace{\sqrt{\frac{|V_n|}{|V_0|}}}_{\text{Occam factor}}\;
  \frac{b_0^{\,a_0}}{b_n^{\,a_n}}\;\frac{\Gamma(a_n)}{\Gamma(a_0)},
  \qquad
  \Prob(h\mid\data)\;\propto\;\Prob(h)\,p(y\mid h).
  \label{eq:evidence}
\end{equation}
All structures use subsets of one full feature vector. The robot therefore keeps the sums
$\Phi^{\!\top}\Phi$, $\Phi^{\!\top}y$, $y^{\!\top}y$ and $n$ once, and reads every structure's posterior and
evidence from their sub-blocks. An update costs $O(d^2)$ per row for $d$ features, whatever the amount of
data.

\paragraph{What the evidence trades off.} For large $n$, $V_n^{-1}\approx n\,\Sigma_\phi$, where $\Sigma_\phi$ is
the mean of $\phi\phi^{\!\top}$. Hence $\log\sqrt{|V_n|/|V_0|}=-\tfrac{d_h}{2}\log n+O(1)$. Also $b_n\approx
\tfrac n2\hat\sigma_h^2$, where $\hat\sigma_h^2$ is the residual variance of the fit. With Stirling's formula for
$\Gamma(a_n)$ (\cref{app:nig}), the log Bayes factor between two structures fitted to the same data is
\begin{equation}
  \log\frac{p(y\mid h)}{p(y\mid h')}\;=\;\frac n2\log\frac{\hat\sigma_{h'}^2}{\hat\sigma_h^2}\;-\;\frac{d_h-d_{h'}}{2}\log n\;+\;O(1).
  \label{eq:bic}
\end{equation}
The first term rewards a better fit and grows linearly in $n$. The second charges each extra coefficient
$\tfrac12\log n$. A spring is adopted only when it explains enough of the motion to pay for its coefficient.
When the data have no spring in them, both fits are equally good and the simpler structure wins.

\subsection{What the data cannot tell apart}

Two columns of $\Phi$ can be proportional on the data even when they are different functions. Then the data
identify only one combination of their weights.

\begin{lemma}[Aliased columns]
\label{lem:alias}
Let columns $x$ and $z$ of $\Phi_h$ satisfy $z=\lambda x$ on every row of the data. Then the likelihood depends
on the weights $(\beta_x,\beta_z)$ only through $\beta_x+\lambda\beta_z$. Along the direction $(\lambda,-1)$ the
posterior equals the prior.
\end{lemma}
\begin{proof}
$\Phi_h\beta_h$ contains $x\beta_x+z\beta_z=x(\beta_x+\lambda\beta_z)$, which does not change when $(\beta_x,\beta_z)$
moves along $(\lambda,-1)$. The likelihood is constant along that line, so Bayes' rule leaves the prior
unchanged there.
\end{proof}

Three cases matter in practice.
(i) A \emph{constant force} on $q_i$ has feature $e_i$, and friction has feature $\sgn(\dot q_i)e_i$. If every
clean step moves $q_i$ the same way, $\sgn(\dot q_i)$ is constant, and only the force plus or minus friction
is identified. A gas strut's preload is an example.
(ii) The \emph{weight on a drawer that must be lifted} is a constant force on the lift. It has the same
problem while the lift only goes up.
(iii) A \emph{cross spring} $k_{12}q_2$ acting on $q_1$ is constant while the hand holds $q_2$ still. It then
looks exactly like a constant force on $q_1$. This is the situation on a door with a lever handle, held
pressed while the door swings.
In each case the evidence \cref{eq:bic} sees the second column as an extra coefficient that buys no fit, and
prefers the structure without it. The fit then puts the load into friction. The cure is motion in both
directions, and the next result says how much is needed.

\begin{proposition}[A reversal separates load and friction]
\label{prop:reversal}
Let $n_+$ clean steps move a coordinate in the positive direction and $n_-$ in the negative direction, and let
the other coefficients be known. For the columns $x=-1$ (load) and $z=-\sgn(\bar v)$ (friction), the Gram
matrix has determinant $4n_+n_-$, and the least-squares estimates have variances
\begin{equation}
  \operatorname{var}(\hat\beta_x)=\operatorname{var}(\hat\beta_z)=\frac{\sigma^2}{4}\Big(\frac1{n_+}+\frac1{n_-}\Big).
  \label{eq:reversal}
\end{equation}
With further unknown columns these variances can only grow. In force units, if $\tau_+$ and $\tau_-$ are the
constant resisting forces measured while moving in the positive and in the negative direction, the load is
$\tfrac12(\tau_++\tau_-)$ and the friction is $\tfrac12(\tau_+-\tau_-)$.
\end{proposition}
\begin{proof}
With $n=n_++n_-$, $x^{\!\top}x=z^{\!\top}z=n$ and $x^{\!\top}z=n_+-n_-$. The determinant is
$n^2-(n_+-n_-)^2=4n_+n_-$, which is positive only if both directions occur. The diagonal of the inverse Gram
matrix is $n/(4n_+n_-)=\tfrac14(1/n_++1/n_-)$, and the least-squares covariance is $\sigma^2$ times this inverse.
By the Frisch--Waugh--Lovell theorem the variance of one coefficient is $\sigma^2$ divided by the squared
residual of its column regressed on the other columns. Adding columns can only shrink that residual, so it
can only raise the variance. The last statement solves $\tau_\pm=p\pm F_c$ for $p$ and $F_c$.
\end{proof}

\Cref{eq:reversal} shows that a short reversal is enough: a few dozen steps in the second direction bring
the variance close to its two-sided value. The next corollary shows what happens without it.

\begin{corollary}[Twice the load]
\label{cor:twice}
Let a lift carry a weight $w>0$ and friction $F_c$, and let the robot have pushed it only upward, so that the
fit puts the whole constant resistance into friction, $\hat F_c=F_c+w$, with no load term. When the robot
lowers the lift slowly, its feed-forward is short of the force needed by exactly $2w$.
\end{corollary}
\begin{proof}
Lowering slowly ($\dot q<0$, $\ddot q\approx0$) needs the upward force $u=w-F_c$, since friction now helps to
hold the lift. The fitted law believes there is no weight and that friction opposes the downward motion, so
its feed-forward is $u=-\hat F_c=-(F_c+w)$. The difference is $(w-F_c)+(F_c+w)=2w$.
\end{proof}
The feedback loop corrects the shortfall only after the lift has dropped. \Cref{cor:twice} caused almost all
failures on drawers that must be lifted (\cref{sec:exp-two}). The fix is the reversal of
\cref{prop:reversal}. Before it first lowers a coordinate it has moved only one way, the robot eases its push
off slowly until the part starts to sink, and lets it sink a short way under a soft speed loop. The same
probe separates a gas strut's preload from friction (\cref{sec:exp-g3}).

\subsection{Stops and locks are events}
\label{sec:events}

A stop $q_{s,i}$ does not appear in the regression. The robot keeps its posterior on a grid. Motion
truncates it: once the coordinate has reached $q_{\max}$, $q_s\ge q_{\max}$. Blocking pins it. When the
coordinate sits at the edge of its explored range and is pushed harder than almost every plausible law needs,
\begin{equation}
  Q_{0.1}\big[\,u-V'_\psi(q)-F_c\,\big]_{\psi\sim p(\psi\mid\data)}>\delta
  \quad\text{for $N_b$ steps},
  \label{eq:blocked}
\end{equation}
the posterior is multiplied by $\N(q_s;q,\sigma_c^2)$. Here $V_\psi$ is the potential of a sampled law and
$Q_{0.1}$ is the 10th percentile over posterior samples. The percentile keeps a part that only sticks, while
friction is still uncertain, from being taken for a stop.

Locks are learned from \emph{opening attempts}: the robot pushes motion $x$ out of its closed zone while the
other motions are at $q_{-x}$. Each hypothesis $c\in\mathcal C_x$ predicts the outcome $r_c(q_{-x})\in\{0,1\}$.
\free{} predicts a release, \jammed{} predicts none, and $(y,\tau)$ predicts a release if and only if $q_y>\tau$.
With a small probability $\epsilon$ of misreading an outcome, the posterior is the prior times
$\prod_{\text{attempts}}(1-\epsilon)^{[o=r_c]}\epsilon^{[o\neq r_c]}$, where $o$ is the observed outcome. The
probability that a configuration releases $x$ is
\begin{equation}
  \Prob(\text{release}\mid q_{-x},\data)\;=\;\sum_{c\in\mathcal C_x}\Prob(c\mid\data)\,r_c(q_{-x}).
  \label{eq:release}
\end{equation}
Near the closed position the push is capped just above what plausible friction needs. Stillness at the cap
is then evidence of a lock, and a sudden release does not launch the part. With a robot body, the cap and the
evidence use the force that arrived at the part, not the force commanded (\cref{sec:ladder}).

%=====================================================================
\section{Priors from a picture}
\label{sec:priors}

\subsection{What a declaration contains}

For each part it sees, the VLM returns a box in the image and its motions. For each motion it gives the kind
(slide or turn) and direction (front--back, up--down, left--right, a vertical or a horizontal axis), a range
guess $\hat r$ with a stated spread $\hat\sigma$, and safety notes. For the part it gives a weight $w_b\in[0,1]$
for each block, read as the probability that the block is in the law, and lock weights $\omega_{xy}\in[0,1]$ for
motions $y$ that may release motion $x$, within the part or on other declared parts. The VLM first lists the
parts it sees, then declares them in batches under a fixed JSON schema. \Cref{fig:declare} shows declarations
on two held-out kitchens.

\subsection{Structure}

We treat the blocks as independent. The prior of a structure is a product of Bernoulli terms with a floor
$\eta$:
\begin{equation}
  \Prob(h\mid d)=\prod_{b\in\mathcal B}\pi_b^{[b\in h]}(1-\pi_b)^{[b\notin h]},\qquad
  \pi_b=\min\big(1-\eta,\ \max(\eta,\ w_b)\big),\qquad \eta=0.05 .
  \label{eq:soft}
\end{equation}
The detent block is one Bernoulli term whose mass is split evenly over the candidate periods. For a part with
two motions the VLM answers per motion (a return spring, a steady force) and for the pair (coupled or not). A
spring on either motion is combined by a noisy-OR, which gives four block probabilities and the prior of each
of the 16 structures. With no declaration, every $\pi_b=\tfrac12$. With $\eta=0$ and weights rounded to 0 or 1,
the declaration becomes a \emph{hard choice} of the hypothesis space. The floor is what makes the declaration
a prior rather than a decision. The next result makes this exact.

\begin{proposition}[The floor bounds the cost of a wrong claim]
\label{prop:overturn}
Let $h$ and $h'$ differ on the set $\Delta$ of Bernoulli blocks. Under \cref{eq:soft}, for any declaration,
\begin{equation}
  \Big|\log\frac{\Prob(h\mid d)}{\Prob(h'\mid d)}\Big|\;\le\;|\Delta|\,\log\frac{1-\eta}{\eta}.
  \label{eq:overturn}
\end{equation}
Hence the posterior prefers $h$ to $h'$, whatever the VLM said, as soon as the Bayes factor
$p(y\mid h)/p(y\mid h')$ exceeds $\big((1-\eta)/\eta\big)^{|\Delta|}$.
\end{proposition}
\begin{proof}
The prior ratio is $\prod_{b\in h\setminus h'}\frac{\pi_b}{1-\pi_b}\prod_{b\in h'\setminus h}\frac{1-\pi_b}{\pi_b}$.
Since $\pi_b\in[\eta,1-\eta]$, each factor lies in $[\eta/(1-\eta),(1-\eta)/\eta]$. Taking logs gives
\cref{eq:overturn}. By Bayes' rule the posterior odds are the prior odds times the Bayes factor.
\end{proof}

With $\eta=0.05$, each disputed block costs a Bayes factor of at most $19$, about $2.9$ nats. A detent
difference costs one more factor, equal to the number of candidate periods. Combining \cref{prop:overturn}
with \cref{eq:bic} gives the number of steps touch needs.

\begin{corollary}[Steps to overturn a claim]
\label{cor:steps}
Let $h$ add the blocks $\Delta$ to $h'$ and reduce the residual variance by the factor
$\hat\sigma_h^2/\hat\sigma_{h'}^2<1$. Then, up to $O(1)$ terms, the posterior prefers $h$ once
\begin{equation}
  n\;\ge\;\frac{2|\Delta|\log\frac{1-\eta}{\eta}+|\Delta|\log n}{\log(\hat\sigma_{h'}^2/\hat\sigma_h^2)} .
  \label{eq:steps}
\end{equation}
\end{corollary}
\begin{proof}
Set the right side of \cref{eq:bic}, with $d_h-d_{h'}=|\Delta|$, above the bound of \cref{eq:overturn} and solve
for $n$.
\end{proof}
The cost of a wrong claim is thus a bounded number of extra steps. With a hard choice ($\eta=0$) the right side
is infinite: a block the VLM missed can never be found. \Cref{sec:exp-g3} shows both behaviours on a
soft-close door whose damper the VLM weighted just below one half.

\paragraph{The first motion.} The floor makes a wrong claim recoverable, but the first motion comes before
any evidence. If the VLM wrongly declares a spring, a feed-forward built on the most probable structure pushes
against a spring that is not there. For the first second of motion the robot therefore feeds forward the
structure that would move the part slowest over the next $0.5$\,rad, among those with at least $0.1$ times the
top posterior, and uses the 10\% quantile of the inertia for its gains.

\subsection{Range}

A VLM sees how far a part \emph{could} travel, not a stop hidden inside that travel, and it states its guesses
too confidently~\citep{groot2024overconfidence}. We keep the guess as the centre and widen the spread:
\begin{equation}
  q_s\sim\N(\hat r,\sigma^2),\qquad \sigma=\max(\hat\sigma,\ \kappa\hat r),\qquad \kappa=0.5,
  \label{eq:calib}
\end{equation}
with $\hat r$ clipped to the range of stops the agent represents. The speed governor of \cref{sec:control}
bounds speed by the 1\% quantile of the stop. With the VLM's own spread (typically $\pm0.3$\,rad) that quantile
lies near the guess, so the robot moves fast toward a stop that may be much nearer. With \cref{eq:calib} the
quantile starts near the closed position and rises only as motion truncates the posterior.

\subsection{Locks}

The lock weights set the prior over $\mathcal C_x$ in \cref{eq:lockset}:
\begin{equation}
  \begin{gathered}
  \Prob(\free)=\pi_f,\qquad \Prob(\jammed)=\pi_j,\qquad \Prob(\unseen)=(1-\pi_f-\pi_j)\,\mu,\\
  \Prob(y,\tau)=(1-\pi_f-\pi_j)(1-\mu)\,\frac{\tilde w_{xy}}{\sum_{y'}\tilde w_{xy'}}\,\frac1{|\mathcal T|},
  \qquad \tilde w_{xy}=\kappa_y\max(\eta,\ \omega_{xy}),
  \end{gathered}
  \label{eq:lockprior}
\end{equation}
with $\pi_f=0.45$, $\pi_j=0.04$, a weight $\kappa_y$ for the kind of releaser (knobs release more often than
doors), and the same floor $\eta$. The parameter $\mu$ is the probability that a lock's releaser is missing from
the declared motions. We measured it in kitchens: with two random views of a development kitchen and Claude as
the VLM, the releaser was missing in 26 of 65 locks, so $\mu=0.40$. With $\mu=0$ the prior is that of a closed
world.

%=====================================================================
\section{Unseen parts and motions}
\label{sec:unseen}

\begin{lemma}[Touch cannot see an unseen releaser]
\label{lem:unseen}
Let every opening attempt on motion $x$ be made with the motions outside $\mathcal K$ at rest. Then for any
sequence of attempts,
\begin{equation}
  \frac{\Prob(\unseen\mid\data)}{\Prob(\jammed\mid\data)}\;=\;\frac{\Prob(\unseen)}{\Prob(\jammed)} .
  \label{eq:ratio}
\end{equation}
\end{lemma}
\begin{proof}
Under \jammed{} no configuration releases $x$. Under \unseen{} the releaser is a motion the robot never moves,
so it stays below its level, and no configuration of the known motions releases $x$ either. The two hypotheses
give every outcome of every attempt the same likelihood, so Bayes' rule multiplies both by the same factor.
\end{proof}

This is a standard case of two faults that no available observation
separates~\citep{sampath1995diagnosability}. The lemma covers both kinds of miss: a releaser on a part the VLM
did not declare, and a motion the VLM did not declare on a part it did, such as the lift of a drawer it
described as only sliding.

\begin{corollary}[Touch cannot certify impossibility]
\label{cor:refuse}
Once \free{} and every hypothesis $(y,\tau)$ with a known allowed releaser are ruled out, the posterior
probability of \unseen{} relative to \unseen{} and \jammed{} together is
$\frac{(1-\pi_f-\pi_j)\mu}{(1-\pi_f-\pi_j)\mu+\pi_j}$, which is $0.84$ with our values. A closed-world agent would
declare the motion jammed. Only a look can change this ratio.
\end{corollary}

\paragraph{An empty look.} If a releaser exists outside $\mathcal K$, one look declares it with probability
$\rho$. A look that shows nothing new has likelihood $1-\rho$ under \unseen{} and $1$ under every other
hypothesis. After $k$ empty looks, starting from $P_0=\Prob(\unseen\mid\data)$,
\begin{equation}
  \Prob(\unseen\mid\data,\ k\text{ empty looks})\;=\;\frac{P_0(1-\rho)^k}{P_0(1-\rho)^k+1-P_0}.
  \label{eq:emptylook}
\end{equation}
In kitchens we measured $\rho=0.28$: the probability that one head-camera look shows a given missed releaser
and the VLM declares it. Bringing $P_0=0.84$ below $\gamma=0.2$ then takes
$k\ge\log\!\big(\tfrac{P_0}{1-P_0}\tfrac{1-\gamma}{\gamma}\big)/\log\tfrac1{1-\rho}\approx 9.3$, that is, ten empty
looks. Within a 30\,s task this is rarely possible, so the refusal rule below also accepts the case that no
further look can be made. A missed motion is different: a closer picture of the part showed it in every case
we tested, so $\rho$ is close to one and a single look settles the question.

\paragraph{The value of a look.} Split the belief into the branch where something undeclared locks the target
and the rest. On the \unseen{} branch the robot cannot succeed without a look, since by \cref{lem:unseen} touch
never finds the releaser. With a look it succeeds if the look reveals the releaser (probability $\rho$) and the
robot then finishes the task with it known (probability $s$). On every other branch a look changes nothing except
its cost $c$, expressed in units of success probability. The gain in success probability is therefore
\begin{equation}
  V_{\mathrm{look}}\;=\;\Prob(\unseen\mid\data)\,(\rho s-0)\;-\;c\;=\;\rho\,s\,\Prob(\unseen\mid\data)-c .
  \label{eq:lookvalue}
\end{equation}
The robot looks when $\Prob(\unseen\mid\data)>c/(\rho s)$. In kitchens, with $c=0.05$ (chosen on development
kitchens), $\rho=0.28$ and $s=0.87$, the threshold is $0.205$. The prior value of \unseen{} is
$0.51\times0.40=0.204$, just below it. The first attempt that finds the motion locked nearly rules out \free{} and
raises $\Prob(\unseen\mid\data)$ to about $0.204/0.55=0.37$, so a look becomes worth asking at once.

\paragraph{Looks take time.} A VLM answers in seconds, while control runs every 10\,ms. The robot does not
wait. It asks, keeps working on its current plan, and keeps holding what it holds; the answer arrives after a
latency $L$. A releaser revealed at time $t+L$ must still be worked in the time left before the budget $T$ ends.
The value of a look asked at time $t$ is then
\begin{equation}
  V_{\mathrm{look}}(t)\;=\;\rho\,s(T-t-L)\,\Prob(\unseen\mid\data_t)-c,
  \label{eq:lookvaluetime}
\end{equation}
where $s(\tau)$ is the success rate with a known releaser and $\tau$ seconds left. We measured $s=0.14$, $0.31$,
$0.52$, $0.77$, $0.84$ and $0.87$ for $\tau=5$, $10$, \dots, $30$\,s. With $T=30$\,s and $L=10$\,s, a look asked at
$t=10$\,s is worth asking only if $\Prob(\unseen\mid\data)>0.05/(0.28\times0.31)=0.58$. A slow VLM must be asked
early or not at all, and beyond some latency no look is worth asking.

\paragraph{When to refuse.} Let $\mathcal A_x$ be the motions the robot may use and $\bar q_y$ the end of motion
$y$'s travel. The robot declares the task on motion $x$ impossible only if
\begin{equation}
  \Prob(\free\mid\data)+\sum_{y\in\mathcal A_x}\sum_{\tau<\bar q_y}\Prob\big((y,\tau)\mid\data\big)<\delta_{\mathrm{inf}}
  \quad\text{and}\quad
  \Big(\Prob(\unseen\mid\data)<\gamma\ \text{ or no look is left}\Big).
  \label{eq:refuse}
\end{equation}
The first condition is the closed-world rule: no known allowed way to release the motion remains. The second is
required by \cref{cor:refuse}. A verdict, once reached, stands until the motion is seen to move. A look that
brings only phantoms counts as empty.

%=====================================================================
\section{Deciding and acting}
\label{sec:control}

\subsection{Where to aim}

Under the posterior the target $q^\star$ is a random variable. It may have several peaks (two candidate detent
periods give two click positions) or be undefined (no click). With tolerance $\varepsilon$, the robot chooses the
decision most likely to be right:
\begin{equation}
  a^\star=\argmax_{a\in\{\text{declare}\}\cup\R}
  \begin{cases}\Prob(q^\star\text{ undefined}\mid\data), & a=\text{declare},\\
  \Prob(|q^\star-a|\le\varepsilon\mid\data), & \text{otherwise.}\end{cases}
  \label{eq:aim}
\end{equation}
We compute it from posterior samples as the candidate with the most samples within $\varepsilon$. The posterior
mean would be a poor aim: halfway between two click positions is certainly wrong.

\subsection{When to experiment}

An \emph{experiment} $e$ is a short closed-loop behaviour, such as a sweep toward the stop. Committing now
succeeds with probability $U_0=\max_a\Prob_b(\text{success}\mid a)$ under the belief $b$. Running $e$ first,
observing its outcome $y$, and then committing has the value
\begin{equation}
  U(e)\;=\;\E_{y}\Big[\max_a\Prob_{b\oplus(e,y)}(\text{success}\mid a)\Big]\;-\;\lambda_t\,\E[\tau_e]\;-\;\lambda_r\,\mathrm{risk}(e),
  \label{eq:voi}
\end{equation}
where $b\oplus(e,y)$ is the updated belief, $\tau_e$ the duration and $\mathrm{risk}(e)$ the probability that the
experiment itself breaks a constraint. The robot runs the best experiment if $U(e)>U_0$.

\begin{proposition}[Information cannot lower the chance of success]
\label{prop:voi}
The first term of \cref{eq:voi} is at least $U_0$.
\end{proposition}
\begin{proof}
$\E_y[\max_a\Prob_{b\oplus(e,y)}(\text{success}\mid a)]\ge\max_a\E_y[\Prob_{b\oplus(e,y)}(\text{success}\mid a)]
=\max_a\Prob_b(\text{success}\mid a)$. The inequality holds because the maximum of expectations is at most the
expectation of the maximum. The equality is the tower property: the belief after the experiment, averaged over
its outcome, is the belief before.
\end{proof}
An experiment is therefore chosen only when the success it adds outweighs its time and risk. For ``open it
fully'', every plausible law has the same target, the stop, and pushing to the stop is the task itself;
\cref{eq:voi} declines every experiment. For ``halfway'', committing aims at half of a vague range guess, while a
sweep to the stop makes the aim precise; \cref{eq:voi} runs the sweep. Unlike expected information
gain~\citep{rainforth2024bed}, \cref{eq:voi} does not value learning the inertia when no instruction depends on
it. We estimate it by simulating each experiment under $S$ posterior samples, giving predicted outcomes $\hat
y_s$, and treating each sample $r$ in turn as the truth, with weights
$w_{rs}\propto\exp(-\|\hat y_r-\hat y_s\|^2/2\sigma_y^2)$ on the other samples.

\subsection{Which releaser to try}

For a latched motion, an option $o$ is ``try the motion now'' or ``drive motion $y$ to a level, then try''. It
releases with probability $p_o$ from \cref{eq:release} and takes expected time $c_o$.

\begin{proposition}[Search order~\citep{stone1975search}]
\label{prop:search}
If the options test disjoint hypotheses and a failed attempt is conclusive, the expected time to release is
minimised by trying options in decreasing order of $p_o/c_o$.
\end{proposition}
\begin{proof}
For an order $\sigma$, attempt $r$ is made only if all earlier ones failed, so
$\E[T_\sigma]=\sum_r c_{\sigma(r)}\big(1-\sum_{r'<r}p_{\sigma(r')}\big)$. Swap two adjacent options, $a$ then $b$,
that follow earlier options with total release probability $P$. Only their two terms change:
$c_a(1-P)+c_b(1-P-p_a)$ becomes $c_b(1-P)+c_a(1-P-p_b)$. The first order is no worse if and only if $c_ap_b\le
c_bp_a$, that is, $p_a/c_a\ge p_b/c_b$. Any order not sorted by this index has an adjacent pair whose swap does
not increase the expected time.
\end{proof}
The robot applies this rule greedily, recomputes the posterior after every attempt, and recurses when the chosen
releaser is itself latched. Between parts, an option drives the releaser to the end of its travel. Within a part
the releaser is the part's other motion, and the options drive it to increasing levels, since the hand that
holds the releaser must also move the latched motion. A wrong lock weight from the VLM only changes $p_o$, and so
the order of the search. The floor in \cref{eq:lockprior} keeps every known releaser in it.

\subsection{Tracking, holding and a speed governor}

Given a velocity set-point $\vref$ and the point estimate $\hat\psi$, the tracking loop applies
\begin{equation}
  u\;=\;u_{\mathrm{ff}}(q,\vref)\;+\;K_p e+K_i\!\int\! e\,dt,
  \qquad e=\vref-\dot q,
  \label{eq:tracking}
\end{equation}
where the feed-forward $u_{\mathrm{ff}}$ is the force the identified law needs to move at $\vref$; for one
coordinate, $u_{\mathrm{ff}}=V'_{\hat\psi}(q)+\hat F_c\tanh(\vref/\epsilon_v)+\hat b\,\vref$. The gains
$K_p=\hat M/T_v$ and $K_i=\hat M/T_v^2$ come from the identified inertia, so the loop responds in about $T_v$
whatever the part weighs. When the law estimate changes, the integrator state is recomputed so that the applied
force does not jump.

For parts with two motions, two more rules apply. First, while the robot moves a latched motion, it \emph{holds}
the releaser where it worked; a hand that let go would let the catch close again. Second, the feed-forward uses
the law averaged over structures,
\begin{equation}
  u_{\mathrm{ff}}(q,\vref)\;=\;\sum_h\Prob(h\mid\data)\,u_{\mathrm{ff},h}(q,\vref),
  \label{eq:average}
\end{equation}
rather than the most probable structure. This force is continuous in the posterior. It matters because of case
(iii) of \cref{lem:alias}. While the handle of a lever door is held still, a cross spring looks exactly like a
constant force on the door, and the data never rule it out. When the door's long motion finally favours the
coupled structure, a feed-forward built on the most probable structure would put the new cross term into the
handle's force at once.

A sampling planner~\citep{williams2017mppi} chooses set-points that reach the target on average over $S$ sampled
laws, and keeps the conditional value at risk~\citep{rockafellar2000cvar} of speed and impact costs small over the
worst 25\% of them. A planner that demands safety under every plausible law stalls in front of an unknown stop,
because some posterior mass always lies just ahead. Following reference governors~\citep{garone2017refgov}, a
governor clips each coordinate's set-point instead. Let $\underline q_s=Q_{0.01}(q_s\mid\data)$ be the 1\% quantile
of the stop, $a_{\mathrm{dec}}$ the deceleration the loop can achieve, and $v_c<v_{\mathrm{imp}}$ a creep speed.
Then
\begin{equation}
  \vref\;\le\;\bar v(q)\;=\;\max\Big(v_c,\ \sqrt{2a_{\mathrm{dec}}\,(\underline q_s-q)_+}\Big).
  \label{eq:governor}
\end{equation}
Along the curve $\dot q=\sqrt{2a_{\mathrm{dec}}(\underline q_s-q)}$ the deceleration is
$\ddot q=\frac{d\bar v}{dq}\,\dot q=-a_{\mathrm{dec}}$. The bound is therefore a braking curve that the loop can
follow.

\begin{proposition}[Impact chance constraint]
\label{prop:governor}
Suppose the loop keeps $\dot q\le\vref+\delta$ and can follow braking at $a_{\mathrm{dec}}$, that
$v_c+\delta\le v_{\mathrm{imp}}$, and that $\underline q_s$ does not decrease during the motion. Then under the
robot's posterior the probability of meeting the stop faster than $v_{\mathrm{imp}}$ is at most $0.01$.
\end{proposition}
\begin{proof}
With posterior probability at least $0.99$, $q_s\ge\underline q_s$. On that event the part meets the stop at some
$q\ge\underline q_s$. There \cref{eq:governor} gives $\vref\le v_c$, so $\dot q\le v_c+\delta\le v_{\mathrm{imp}}$.
Before $\underline q_s$ the set-point follows the braking curve, which the loop can track, so the part arrives at
$\underline q_s$ no faster than $v_c+\delta$. A fast impact therefore requires $q_s<\underline q_s$, an event of
probability at most $0.01$.
\end{proof}
The guarantee holds under the robot's own posterior; it does not cover arbitrary model error. Motion only
truncates the stop posterior, which raises $\underline q_s$, so the monotonicity assumption holds while the part
moves toward the stop.

%=====================================================================
\section{From the simulator to cameras and touch}
\label{sec:ladder}

The agent of \cref{sec:ident,sec:priors,sec:unseen,sec:control} needs four things from the world: which part
each declared box is, where to grasp it, how its handle moves when its coordinates move (its joint), and the
part's coordinates while it works. In simulation it is easy to take any of these from the simulator without
noticing. \Cref{tab:ladder} states, for each experiment, where each one comes from. Each row removes something
the row above took from the simulator.

\begin{table}[tbp]
\centering
\small
\caption{What the robot takes from the simulator in each experiment. ``Sim.'' means the simulator gives it.
Every experiment is scored by the simulator.}
\label{tab:ladder}
\setlength{\tabcolsep}{4pt}
\renewcommand{\arraystretch}{1.15}
\begin{tabular}{@{}>{\raggedright\arraybackslash}p{3.5cm} *{6}{>{\raggedright\arraybackslash}p{1.75cm}}@{}}
\toprule
Experiment & Parts, motions & Which part a box is & Handle & Joint & Coordinates & Floor map \\
\midrule
Core agent & given & n/a & n/a & sim. & sim. & n/a \\
Kitchens, objects, looks & VLM & sim.\ pixels & n/a & sim. & sim. & n/a \\
Two motions, mixed scenes & VLM & sim.\ pixels & n/a & sim. & sim. & n/a \\
Physical robot & VLM & sim.\ pixels & sim. & sim. & sim. & sim. \\
Handles by depth, joints by touch & VLM & sim.\ pixels & depth camera & touch & from hand & sim. \\
Nothing from the simulator & VLM & the box itself & depth, gripper camera & touch & from hand & own depth map \\
\bottomrule
\end{tabular}
\end{table}

\paragraph{Acting through a body.} In the first four rows the agent's force goes straight onto the part's
coordinates. With a robot arm, a push can go nowhere because the grip slipped, and a latch would produce the same
observation; in our first robot runs this made the agent call free doors jammed. The remedy is to condition on
what the body achieved. The law is fitted to the force that the wrist sensor says arrived at the part, projected
through the part's Jacobian, not to the force commanded. A failed opening attempt counts as evidence of a lock only
in proportion to how far the arrived force exceeded what a free part would need. A release counts only if no slip
was reported. Without a body these rules are inactive, so the agent is unchanged.

\paragraph{Handles from depth.} Inside the VLM's box, the depth image gives the part's face and the handle, a long
thin blob standing out from it. Depth carries the noise of a stereo camera, about 2.5\,mm at 1\,m and 1.3\,cm at
3\,m. At 3\,m this noise hides most handles, so the robot drives to about 1\,m in front of the part and looks again
with a pan--tilt camera on its head.

\paragraph{Joints by touch.} The robot's first model of a joint is a guess from the VLM's words and the depth
image: ``slides front--back'' is a slide into the face; a door's hinge is at the edge of the face farther from the
handle. After it grips, the robot pushes the handle a little along each motion with a soft hand and records where
the hand goes. A straight path is a slide. A path that bends is a hinge, and the bend gives its axis and position.
While it works, the robot reads the part's coordinates from its own hand through this model.

\paragraph{Nothing from the simulator.} The last row removes the rest. A part is one of the VLM's boxes, not a
simulated part. The robot plans where to stand on its own model of the part, for the range such a motion usually
has (a door about 80$^\circ$, a drawer about 30\,cm); the agent then finds the real stops by touch. The floor and
the obstacles come from the robot's own map, built by counting depth points from fixed room cameras and its head
camera in 2.5\,cm cells. The base must stay clear of seen cells, and the arm may not pass through them except at
the handle. A depth camera on the wrist, beside the fingers, guides the grip: with the hand lined up 27\,cm and then
14\,cm in front of the handle, it finds the handle where the model puts it and moves the grip onto it, across the
bar and in depth (\cref{fig:cameras}). The agent is told only what the words say: how many motions, and whether
each turns both ways from rest. The fingertips feel that they touch something, not what. Two failures appeared
only here, both on a door with a lever handle, and led to two rules. Held against its stop, the lever looked to the
law like a stiff spring, so the force the law said it needed grew past the most the robot may push; for the
robot's own parts, a push at 80\% of the most force with no motion is now a stop. At a wide swing the stretched
arm delivered little of the force asked for; asking for full force while the wrist feels little is now the end of
the arm's reach. The tap needed a change of grip. Pinched at its ball tip with the hand's orientation fixed, the ball
rolled between the fingertips as the lever turned, and the robot's reading of the lever crept. The robot now holds the
tip firmly and turns its hand with the lever, as it does on a door, with light pushes (at most 8\,N) for a lever's tip.
Because the hand reads the part's speed late, moves of the robot's own parts are planned at 0.75 of the usual speed.

\begin{figure}[tbp]
\centering
\includegraphics[width=0.46\linewidth]{\sroot 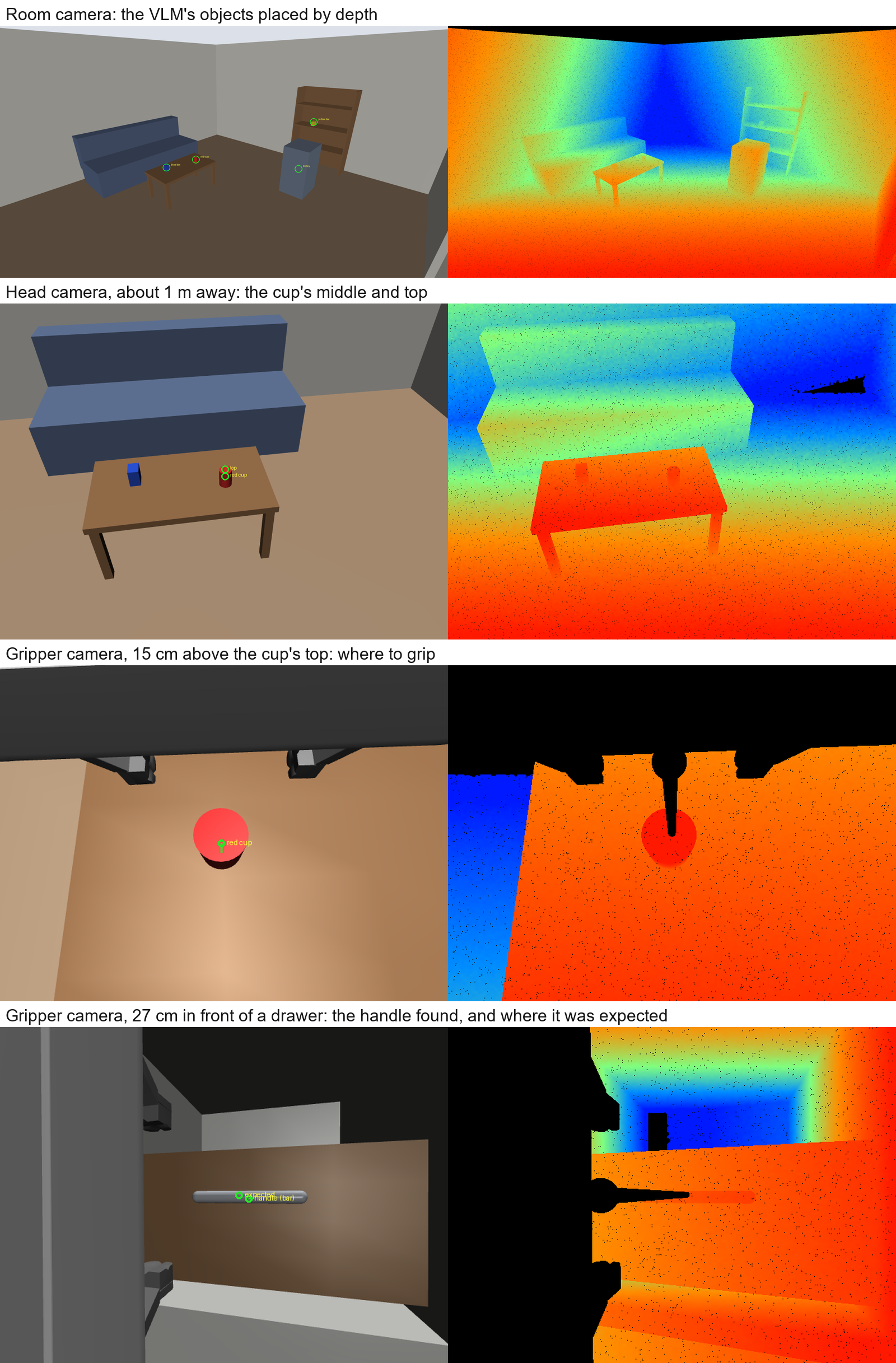}
\caption{What the robot's cameras see when it takes nothing from the simulator. Left: colour, with what the robot
found marked in green. Right: depth, near in red and far in blue, black where there is no reading or where the
robot's own body is. Top to bottom: a room camera, the head camera about a metre from a cup, the gripper camera
above the cup, and the gripper camera in front of a drawer handle.}
\label{fig:cameras}
\end{figure}

%=====================================================================
\section{Experiments}
\label{sec:experiments}

\paragraph{Protocol.} All experiments run in MuJoCo~\citep{todorov2012mujoco}. Hidden laws are injected into the
parts. Kitchens are RoboCasa scenes~\citep{nasiriany2024robocasa}. The VLM is Claude (Opus~5.5), queried through
a file queue whose answers are cached by request. Reruns are therefore exact, but latency is simulated, not
measured. For every experiment we developed on one set of scenes, froze the agent, the prompts and every
parameter, rendered new scenes, and ran each condition once. Paired comparisons use a sign test on the tasks where
two conditions disagree. Tasks of one scene share its views and parts, so we also count scenes where it matters.
Only perception, priors or the body change between conditions; the identifier and the controller stay the same.
The parameters are listed in \cref{app:params}.

\subsection{The core agent with known parts}
\label{sec:exp-core}

We first check the identifier and the decision rules when every part is known. On \emph{single valves} (40
held-out valves, six instructions each, 12\,s episodes), each valve has a spring with probability one half,
detents with one of four periods with probability one half, and a stop in $[1,3]$\,rad (\cref{tab:core}, left).
L2A succeeds in 98\% of episodes with no violations. Without experiments success drops to 61\%, because nothing in
``halfway'' reveals the stop. L2A runs 1.10 experiments per episode on average, about one sweep for ``halfway''
and none for ``fully'', as \cref{eq:voi} predicts. Without robust control (certainty equivalence) the agent breaks a
constraint in 20\% of episodes, mostly by hitting the stop too fast. A policy that imitates a teacher who knows the
law reaches only 36\%. Every ``no click'' refusal by L2A follows an experiment.

On \emph{lockbox cabinets} (30 held-out cabinets, 1,104 episodes), main parts are latched by knobs and levers, in
chains of up to four, with jammed parts and forbidden releasers (\cref{tab:core}, right). L2A succeeds in 96\%. A
language-model planner~\citep{huang2022innermonologue} that sees the same outcomes as text and chooses the next
attempt succeeds in 83\% (L2A wins 26 instances and loses 2, $p=3\times10^{-6}$). It cannot tell a releaser not yet
found from an impossible task, and it refuses correctly in only 14 of 32 impossible cases. An agent without a lock
model succeeds in 72\%.

\begin{table}[tbp]
\centering
\small
\caption{The core agent with known parts, held-out sets. Left: single valves (1,440 episodes). Right: lockbox
cabinets (1,104 episodes); ``Refused'' counts impossible tasks correctly refused.}
\label{tab:core}
\begin{tabular}{@{}l c c@{}}
\toprule
Valves & Success & Violations \\
\midrule
True law & 100\% & 0\% \\
Prior only & 53\% & 0\% \\
Imitation of a teacher & 36\% & 35\% \\
Certainty equivalent & 49\% & 20\% \\
Robust, no experiments & 61\% & 0\% \\
\textbf{L2A} & \textbf{98\%} & \textbf{0\%} \\
\bottomrule
\end{tabular}
\hfill
\begin{tabular}{@{}l c c@{}}
\toprule
Cabinets & Success & Refused \\
\midrule
True law & 100\% & 32/32 \\
\textbf{L2A}, language prior & \textbf{96\%} & 28/32 \\
L2A, uniform lock prior & 96\% & 30/32 \\
Language-model planner & 83\% & 14/32 \\
No lock model & 72\% & 32/32$^\dagger$ \\
\bottomrule
\end{tabular}

\smallskip
{\footnotesize $^\dagger$It refuses every blocked part, and so also fails 50 achievable tasks.}
\end{table}

\subsection{Declarations from images, in kitchens never seen}
\label{sec:exp-kitchens}

Twenty RoboCasa kitchens (114 tasks) were used for development. Twenty others (117 tasks) were rendered only after
the prompts and the calibration rule were frozen. Each task is one instruction on one part. Some parts are jammed,
and then success means refusing.

\begin{figure}[tbp]
\centering
\includegraphics[width=0.92\linewidth]{\sroot 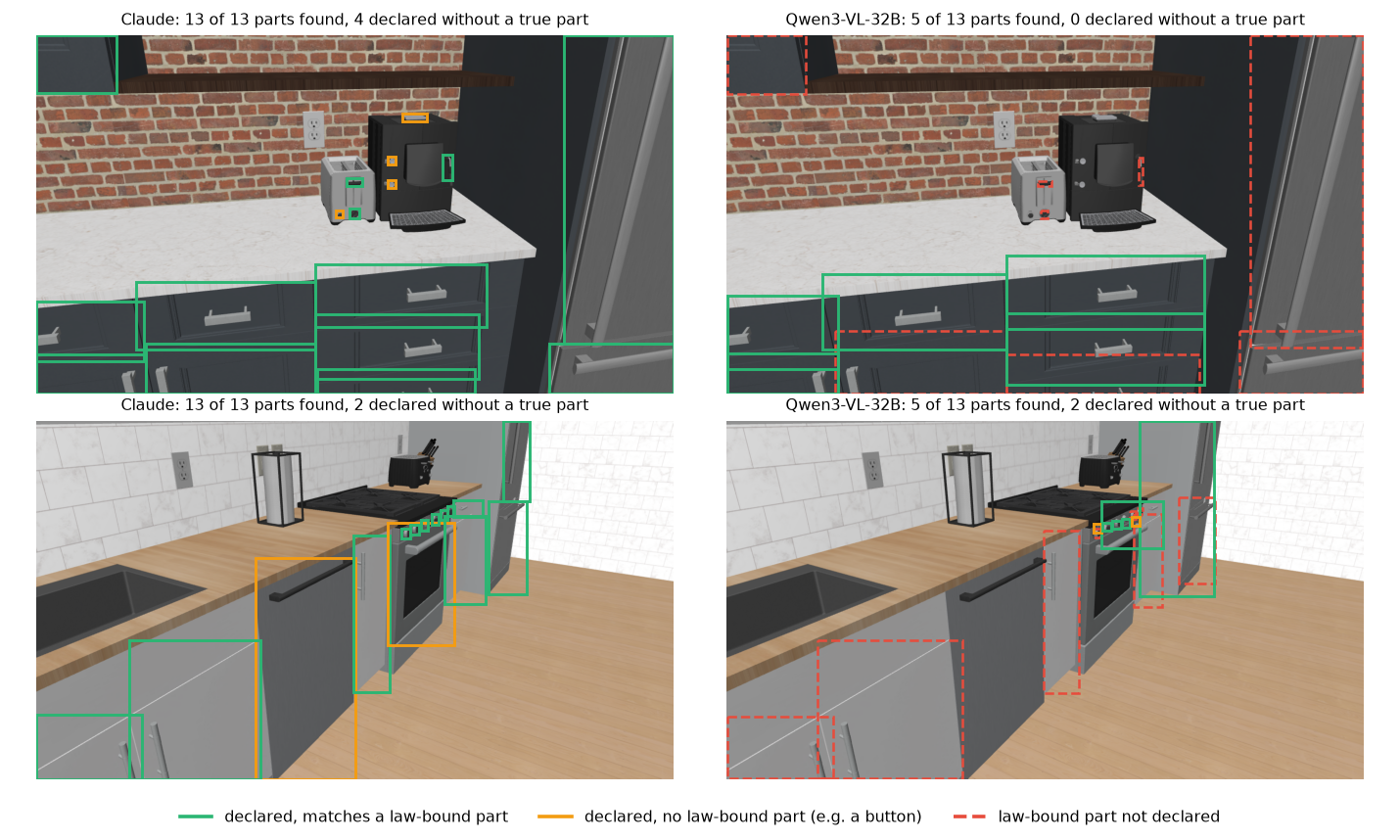}
\caption{Declarations on two held-out kitchens. Green: a declared part that matches a law-bound part. Orange: a
declared part that matches none, usually a real part that cannot move. Dashed red: a law-bound part that was not
declared. Claude finds all 13 parts in both views; Qwen3-VL-32B finds 5 in each.}
\label{fig:declare}
\end{figure}

\paragraph{The VLM's declarations are informative.} On held-out kitchens Claude finds 81--88\% of the law-bound
parts; 80--84\% of its declared parts are real; and it names the right kind of motion for 92--94\% of the parts it
finds. Its block weights have a Brier score of 0.16--0.17, against 0.25 for no information. These match the
development scores (0.81, 0.87, 0.94 and 0.15), so the prompts were not tuned to the development scenes. An open
model run locally, Qwen3-VL-32B, finds only 29--39\% of the parts.

\begin{table}[tbp]
\centering
\small
\caption{Kitchen tasks with Claude's declarations. Successes over all tasks; in brackets, achievable tasks solved
and impossible tasks refused.}
\label{tab:kitchens}
\begin{tabular}{@{}l c c@{}}
\toprule
Condition & Development (114) & Held-out (117) \\
\midrule
Every law-bound part known & 113 & 115 (100/102, 15/15) \\
\midrule
One view per part, VLM's raw ranges & 57 & n/a \\
One view per part, VLM's numbers ignored & 80 & n/a \\
\textbf{One view per part, calibrated ranges} & \textbf{91} & \textbf{94} (81/102, 13/15) \\
Two random views, calibrated & 71 & 65 (50/102, 15/15) \\
Two random views and one look at the target & 76 & 86 (71/102, 15/15) \\
\bottomrule
\end{tabular}
\end{table}

\paragraph{Declarations drive the agent, once the ranges are calibrated.} With one view per part, the VLM's
declarations give 80\% success on both development and held-out kitchens, against 98\% when every part is known
(\cref{tab:kitchens}). Nearly all of the gap comes from parts the VLM did not declare, most often a releaser. Taken
at face value the VLM's numbers hurt. Raw ranges gave 57 successes, fewer than ignoring the numbers (80), because
the robot trusted a far stop and hit a near one. The calibrated prior \cref{eq:calib}, fixed on development, raised
this to 91 and carried over unchanged to the held-out kitchens. With fewer views, more releasers are missed, and
success falls to 65.

\subsection{New physics from declared blocks}
\label{sec:exp-g3}

The kitchen laws use blocks the agent already had. Here we test blocks it did not have: gravity load, constant
force (preload) and end damper. Five MuJoCo object families have real mass under real gravity: a lid that falls
shut, a flap that falls open, a flap lifted by gas struts (gravity and preload), a soft-close door (end damper) and
a self-closing door (spring). Each family has 20 randomly sized objects and four instructions (fully or halfway,
normal or gently), with a 12\,s budget. The VLM saw one picture of each object. It declared gravity for the lid and
the flap, and gravity with preload for the strut flap, all correctly. It gave the soft-close damper a weight of
$0.45$, just below one half. Two agent limits found on development objects were fixed before the test: a breakaway
ramp for heavy parts at rest, and the reversal probe of \cref{prop:reversal} for preload.

\begin{table}[tbp]
\centering
\small
\caption{New physics on 100 fresh objects (20 per family, 4 instructions). Successes out of 80 per family.
``Hard'': blocks with weight $\ge0.5$ form the hypothesis space. ``Soft'': the floored prior \cref{eq:soft}.}
\label{tab:g3}
\begin{tabular}{@{}l c c c c c c@{}}
\toprule
Hypothesis space & Lid & Flap & Strut & Soft-close & Self-closing & Total (400) \\
\midrule
Original blocks only & 60 & 60 & 80 & 73 & 80 & 353 \\
All blocks, no declaration & 69 & 73 & 65 & 72 & 74 & 353 \\
VLM per family, hard & 79 & 80 & 66 & 73 & 80 & 378 \\
VLM per family, soft & 79 & 80 & 64 & 80 & 80 & \textbf{383} \\
VLM per object, hard & 79 & 80 & 66 & 76 & 80 & 381 \\
VLM per object, soft & 78 & 78 & 64 & 79 & 79 & 378 \\
True blocks & 79 & 80 & 66 & 80 & 80 & 385 \\
True law & 80 & 80 & 80 & 80 & 80 & 400 \\
\bottomrule
\end{tabular}
\end{table}

\Cref{tab:g3} gives the result. The VLM's declarations raise success from 353 to 378--383 of 400, close to the 385
of a perfect declaration. Against the original blocks they win 45 episodes and lose 20 ($p=0.003$); counted by
object, the weakest of the four uses still wins 33 objects and loses 16 ($p=0.02$). Allowing every block without a
declaration does not help (353): it runs the preload probe on objects with no preload and spends time on structures
the object does not have. Declarations help because they narrow the space.

The soft-close door shows \cref{prop:overturn,cor:steps} at work. As a hard choice, the damper at weight $0.45$ is
left out and the door loses 7 episodes. As a soft prior, the damper starts at $0.45$, the closing motion confirms
it, and the door succeeds in all 80 episodes, with the damper in the most probable structure every time (7 won, 0
lost, $p=0.016$). The strut flap shows the cost of \cref{lem:alias}. The original law folds the strut's preload into
friction. That law is wrong, but it is good enough to open the flap. The declared space separates preload from
friction, and under ``gently'' with heavy friction the probe and the search for the stop use up the budget (11 of
the 14 failures). Identifying the right model takes time that the agent does not yet weigh when it chooses
experiments.

\subsection{Unseen parts and looking again}
\label{sec:exp-looks}

We tested the look rule of \cref{sec:unseen} on two sets of 20 fresh kitchens (242 tasks), rendered and declared
after the rule was frozen. The robot starts from two random views. It asks for a look by \cref{eq:lookvalue} while
its hands keep working, and also whenever it would otherwise refuse with $\Prob(\unseen\mid\data)>\gamma=0.2$. A look
shows a head-camera view. The latency $L$ was fixed in advance at 10\,s, with 2\,s as a check.

\begin{table}[tbp]
\centering
\small
\caption{Looking again, on two fresh sets of 20 kitchens pooled (242 tasks, 36 impossible). Sign tests are against
one free look.}
\label{tab:looks}
\begin{tabular}{@{}l c c c@{}}
\toprule
Condition & Success & Refused & Against one free look \\
\midrule
One view per part (reference) & 196 & 36/36 & \\
Two random views, no look & 169 & 36/36 & \\
Two views, one free look at the target & 188 & 36/36 & \\
Look rule, answers after 10\,s & 177 & 36/36 & 4 won, 15 lost ($p=0.02$) \\
\textbf{Look rule, answers after 2\,s} & \textbf{199} & 36/36 & 11 won, 0 lost ($p=0.001$) \\
\bottomrule
\end{tabular}
\end{table}

\Cref{tab:looks} gives the result. It depends on the latency, as \cref{eq:lookvaluetime} predicts. With answers
after 2\,s the look rule is the best condition, above even the reference that starts from a view of every part,
though not significantly. With answers after 10\,s it is below a single free look: a releaser found late leaves too
little of the 30\,s budget to use it. Counted by kitchen, the effects are not significant (5 kitchens against 0 at
2\,s, $p=0.06$; 1 against 5 at 10\,s). This is a trend, not a firm result. The look rule refuses every impossible
task. A rule-based loop that reset its search after each look refused only 9 of 15 in the first fresh set. Using
the time-aware value \cref{eq:lookvaluetime} instead of \cref{eq:lookvalue} changed one task in 117. Since
$\rho=0.28$, most releasers that the first views missed stay missed however the looks are timed.

In small cabinet scenes with a visible catch, the VLM usually named the catch first. On a second held-out set it
failed to declare the catch in 4 of 24 such scenes. In those 4 scenes one view solves 2 of 20 tasks, and the frozen
look rule solves 18 (answers after 2\,s): a look shows the catch, and the robot grounds and opens it. This follow-up
is targeted, since we chose the set because it had misses.

\subsection{Parts that move in two ways, and mixed scenes}
\label{sec:exp-two}

\begin{figure}[tbp]
\centering
\includegraphics[width=0.92\linewidth]{\sroot 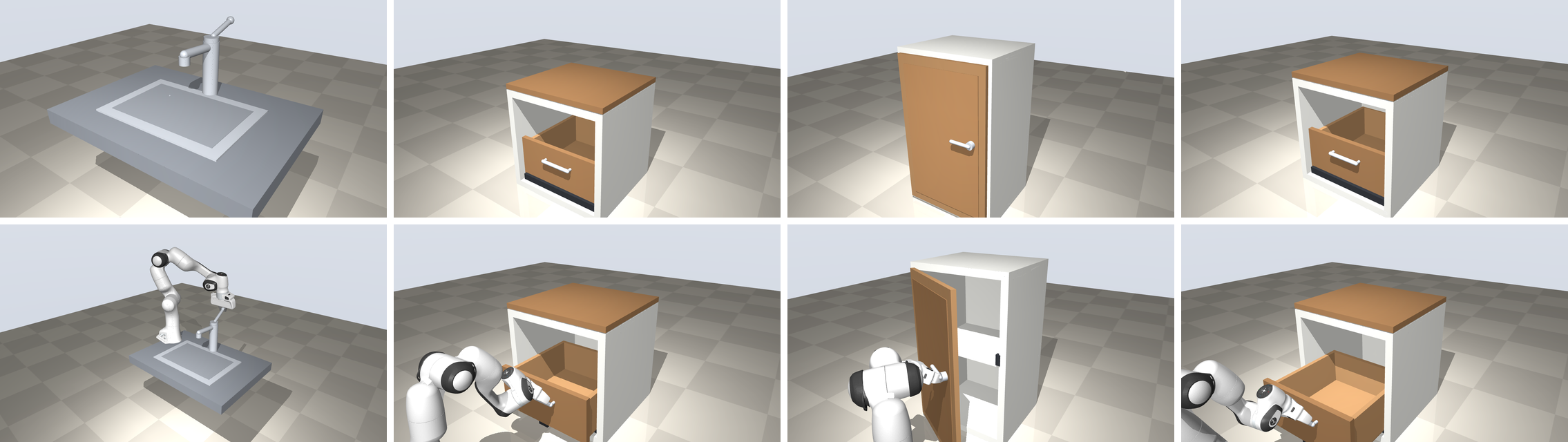}
\caption{Parts with two motions. Top: the one picture the VLM saw. Bottom: the robot working the part. From left:
mixer tap, drawer that must be lifted before it slides, door with a lever handle, drawer that sags.}
\label{fig:twomotions}
\end{figure}

\paragraph{The VLM's physics from one picture.} We first told the VLM the two motions in words and asked only for
the physics (\cref{fig:twomotions}). On 80 held-out parts (320 tasks), its answer raised success from 245 with no
starting belief to 269 (\cref{tab:two}; 28 tasks won, 4 lost, $p=2\times10^{-5}$). A hand-written guess did not help
(243). Most of the gain is on the drawer that must be lifted (17 won, none lost). The VLM said that the lift carries
weight, which touch cannot tell early, by \cref{lem:alias}. Even a perfect guess lost 37 tasks, nearly all by
\cref{cor:twice}: knowing that the lift carries weight is not knowing how much.

\begin{table}[tbp]
\centering
\small
\caption{Two-motion parts, 80 held-out parts, 80 tasks per kind. Only the starting belief changes.}
\label{tab:two}
\begin{tabular}{@{}l c c c c c@{}}
\toprule
Starting belief & Tap & Lift drawer & Lever door & Sagging drawer & All (320) \\
\midrule
None (flat) & 80 & 45 & 79 & 41 & 245 \\
Hand-written & 80 & 46 & 76 & 41 & 243 \\
\textbf{VLM, one picture} & 80 & \textbf{62} & \textbf{80} & 47 & \textbf{269} \\
Perfect (true blocks) & 80 & 62 & 80 & 61 & 283 \\
True law & 80 & 80 & 80 & 80 & 320 \\
\bottomrule
\end{tabular}
\end{table}

\paragraph{Mixed scenes: the VLM also declares the motions.} In a mixed scene the VLM must decide how each part
moves. Each scene is a row of five units: a tap, two drawers and two doors, each of a random kind
(\cref{fig:mixed}). For each part the VLM gives a box, its motions (kind and direction) and the physics. The robot
works each part with exactly the motions the VLM named. When a motion is latched and no declared motion releases it,
\unseen{} is all that is left, and the robot asks for a \emph{second look}: a closer picture of the part and the list
of motions that failed. We ran 20 held-out scenes (100 parts, 366 tasks; \cref{tab:mixed}).

\begin{figure}[tbp]
\centering
\includegraphics[width=0.92\linewidth]{\sroot 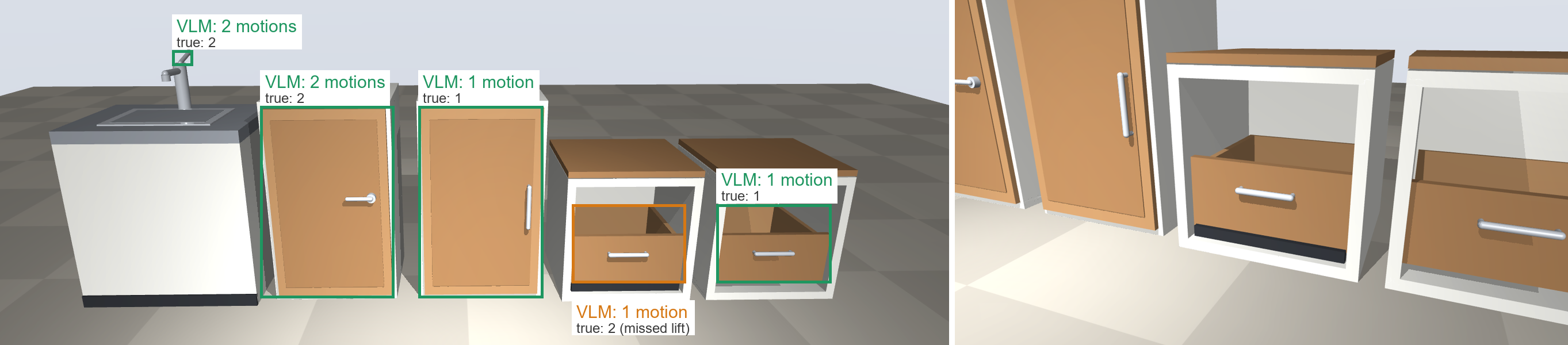}
\caption{A held-out mixed scene. Left: the VLM's boxes; green, the right number of motions; orange, one motion for a
drawer that must first be lifted. Right: the closer picture of the second look, in which the VLM found the lift.}
\label{fig:mixed}
\end{figure}

\begin{table}[tbp]
\centering
\small
\caption{Mixed scenes. Left: 20 held-out scenes (366 tasks). Right: the load-measuring reversal of
\cref{prop:reversal} on new held-out sets, without and with it.}
\label{tab:mixed}
\begin{tabular}{@{}l c@{}}
\toprule
Condition & Success (366) \\
\midrule
True law & 366 \\
True motions, flat prior & 317 \\
Each motion as its own part & 228 \\
VLM, no second look & 284 \\
\textbf{VLM, with second look} & \textbf{329} \\
\bottomrule
\end{tabular}
\hfill
\begin{tabular}{@{}l c c@{}}
\toprule
New held-out set & Without & With \\
\midrule
Two-motion parts, flat prior (320) & 253 & 283 \\
Two-motion parts, VLM (320) & 272 & \textbf{290} \\
Mixed, true motions, flat prior (363) & 320 & 344 \\
Mixed, VLM with second look (363) & 297 & \textbf{313} \\
\bottomrule
\end{tabular}
\end{table}

The VLM found all 100 parts and gave the right number of motions for 73. It gave one motion, the slide, for all 27
drawers that must be lifted first; the only sign of the hook is a dark strip under the drawer front. The robot
pulled, found the slide latched with no declared releaser, and asked again. In the closer picture the VLM found the
lift in all 27 cases, and in 74 of 74 over three mixed test sets. The second look won 45 tasks and lost none
($p=6\times10^{-14}$; 12 scenes better, none worse). Treating the two motions of a part as one law matters even more:
working each motion as its own part solved 228. With the second look every task reached its goal. All 37 failures
broke a safety rule, the same lift lowered before its load was known.

\paragraph{Feeling the load.} We then added the reversal of \cref{prop:reversal} before the first lowering, froze it,
and rendered new held-out sets (\cref{tab:mixed}, right). It helps every starting belief: the flat prior gains 30
tasks on two-motion parts and the VLM 18. It also shrinks the VLM's advantage (290 against 283; 12 won, 5 lost, not
significant), because the robot now measures the load instead of needing the guess. In the new mixed scenes the VLM
left the tap out of 10 scenes, calling it possibly decoration, so 40 tasks were lost before any touch. Without the
taps, the VLM with a second look solved 273 of 283 tasks.

\subsection{A physically simulated robot}
\label{sec:exp-robot}

Now every part is a set of MuJoCo joints carrying its law, and the robot is RoboCasa's mobile Panda: a wheeled base,
a torso lift, a seven-joint torque-controlled arm and a two-finger gripper, all colliding with the furniture
(\cref{fig:body}). The robot chooses where to stand, grasps the handle by friction, and turns the agent's force into
a hand force through the part's Jacobian, corrected by the wrist sensor. Four body faults were fixed on four
development scenes; for example, the arm's joint friction ate most of a door's push. We then froze the system and
ran 20 new mixed scenes once (366 tasks).

\begin{figure}[tbp]
\centering
\includegraphics[width=0.92\linewidth]{\sroot 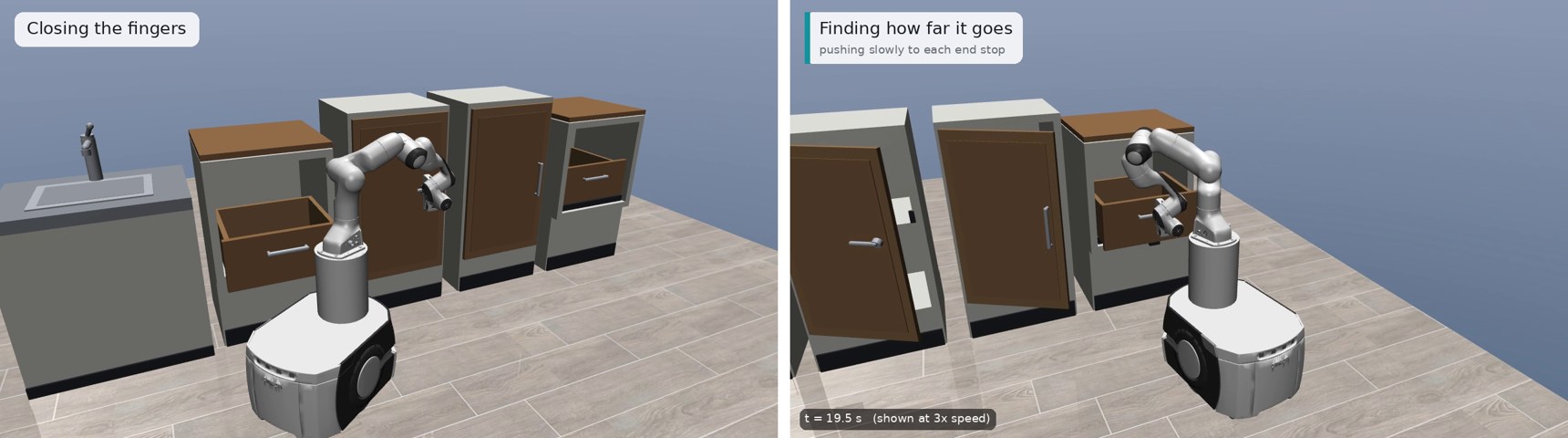}
\caption{The physical robot on a held-out scene. Left: it grasps a door's lever handle; the drawer it opened earlier
is still open. Right: a drawer that must be lifted, lifted and pulled out.}
\label{fig:body}
\end{figure}

\begin{table}[tbp]
\centering
\small
\caption{The physical robot on 20 new mixed scenes (366 tasks: 102 one-motion, 264 two-motion). ``No body'': the
same agent and VLM answers with the force put on the part directly.}
\label{tab:robot}
\begin{tabular}{@{}l c c c@{}}
\toprule
Starting belief & All & One motion & Two motions \\
\midrule
True motions, flat prior (robot) & 238 & 51 & 187 \\
\textbf{VLM with second look (robot)} & \textbf{247} & 51 & 196 \\
True law (robot) & 326 & 70 & 256 \\
VLM with second look, no body & 354 & 102 & 252 \\
\bottomrule
\end{tabular}
\end{table}

The grasping itself worked. In 1,098 episodes the robot never failed to grasp a handle and lost its grip once
(\cref{tab:robot}). It reached each part's full travel in 319 of 366 tasks; the other 47 were doors that swing too
far to follow from any place where the base clears the door. The VLM was again useful but not decisive (26 tasks won
and 17 lost against the flat prior, not significant). The main cost is the body. The same learner solved 354 tasks
without a body and 247 through the arm, whose inertia and friction sit between the agent's force and the part.

\subsection{Handles by depth, joints by touch, and nothing from the simulator}
\label{sec:exp-own}

\paragraph{Handles by depth, joints by touch.} On 50 parts in 10 new scenes, declared by the VLM after the handle
detector was frozen, the head camera placed the grip within 0.56\,cm of the handle (median; 90\% within 1.08\,cm), and
49 of 50 grips held. The joints found by touch were right in kind for all 68 motions, with the axis within
0.1$^\circ$ and a hinge within 1.1\,cm (medians). On the same 10 scenes (179 tasks), run once with both conditions
paired, the robot completed 53 tasks with its own joints and 69 with the simulator's (13 won, 29 lost, $p=0.02$).
Most of the loss is on taps, whose lever the hand pinches by its ball tip; the ball creeps between the fingers.

\paragraph{Nothing from the simulator.} We developed the last row of \cref{tab:ladder} on scenes used before (the
tap on 12 of them), froze it, wrote the scoring rule down, and then rendered 10 new scenes (50 parts). The VLM's answers
were the true list of parts and motions. In one session per scene, the robot opened each part halfway (the tap:
halfway and hot) and closed it again. We ran the same requests with the row above it, which takes handles from depth
and joints from touch but asks the simulator which part each box is and uses the simulator's floor map. An opening
counts when the main motion ends at 40--60\% of its true travel (the tap also needs its swing at 90\% or more of its
travel), a closing when it ends at 5\% or less; both need the true top speed at most 0.5 and impact at most 0.3.
\Cref{tab:own} gives the result (first test).

\begin{table}[tbp]
\centering
\small
\caption{The robot that takes nothing from the simulator (own), against the row above it in \cref{tab:ladder} (sim.:
told which part each box is, the simulator's floor map). Both find handles by depth and joints by touch. Two held-out
tests of 10 new scenes each (100 requests per condition), each run once; the second after the robot learned to look
again at a part it moved before. Closings scored by the second rule (see text).}
\label{tab:own}
\begin{tabular}{@{}l c c c c@{}}
\toprule
 & \multicolumn{2}{c}{First test} & \multicolumn{2}{c}{Second test} \\
\cmidrule(lr){2-3}\cmidrule(lr){4-5}
Requests & own & sim. & own, looks again & sim. \\
\midrule
Open halfway, all parts (50) & 19 & 14 & 20 & 14 \\
\quad of which the tap (10) & 5 & 0 & 7 & 1 \\
\quad of which cabinet and lever doors & 0 of 20 & 0 of 20 & 0 of 20 & 0 of 20 \\
Close (50) & 34 & 49 & 50 & 48 \\
\midrule
All (100) & \textbf{53} & \textbf{63} & \textbf{70} & \textbf{62} \\
Done by one condition only & 9 & 19 & 13 & 5 \\
\bottomrule
\end{tabular}
\end{table}

In the first test, taking nothing from the simulator cost tasks: 46 against 61 by the rule fixed in advance (10 won,
25 lost, $p=0.02$). It gained on the tap, which it opened in 5 of 10 scenes and the other robot, still pinching the
ball tip, in none. It lost on closing. After a failed attempt to open a cabinet door, the door closed by itself when the
robot let go, but the robot still believed it open and reached for the handle in the wrong place (13 misses). The
other robot is told which part the box is and finds the handle again. Seven of its lost closings, and two of the
other robot's, were parts never opened that the robot rightly called closed, which that rule counts as failures. We
therefore fixed a second rule before the next test: a closing also counts when the robot says the part is already
closed and it is. By the second rule the first test reads 53 against 63 ($p=0.09$).

We then let the robot keep track of a part it has moved. Before it works such a part again, it walks up and looks with
its head camera along the part's path, from its last reading to closed; the handle it finds gives the part's position
through the joints it felt. If it sees the part closed and is asked to close it, it says so. On development scenes this
raised closings from 26 to 30 of 30. We froze it and ran a second test on 10 further new scenes (\cref{tab:own}, second
test). The robot that takes nothing from the simulator now completed 70 requests against 62 (13 won, 5 lost, $p=0.10$;
by the first rule 67 against 62). It closed every part (50 of 50) and opened the tap in 7 of 10 scenes. The head
camera found the handle in 33 of its 47 looks; when it did not, the robot went by its last reading. Neither robot opened
a cabinet door or a door with a lever handle halfway in either test: from a base that stands still, the arm follows a
wide door only part of the way.

%=====================================================================
\section{Related work}
\label{sec:related}

\paragraph{Vision, language and action.} Vision--language--action
models~\citep{brohan2023rt2,kim2024openvla,octo2024,black2025pi0} learn actions from demonstrations, and force-aware
variants~\citep{yu2025forcevla,huang2025tactilevla} add touch to the input. Language and vision models also plan, or
write constraints for a controller~\citep{ahn2022saycan,huang2022innermonologue,huang2023voxposer,huang2024rekep}.
These methods act on what the image shows. Our agent uses the VLM only for priors and keeps the hidden physics as an
explicit posterior that touch updates.

\paragraph{Physical properties from vision and interaction.} Physical properties have been estimated from
video~\citep{wu2015galileo} and with language models~\citep{gao2024physobjects}. Phys2Real~\citep{wang2026phys2real},
CoRAL~\citep{cicek2026coral} and RigPI~\citep{he2026rigpi} refine VLM estimates of mass, friction or centre of mass
by interaction, and JODA~\citep{gao2026joda} lets a VLM choose joint effects from a library. These methods refine
numbers. Our VLM proposes the \emph{structure} of the law and the set of parts, motions and locks, with floors that
make every claim recoverable (\cref{prop:overturn}), and the agent keeps a hypothesis for what the VLM did not see.

\paragraph{Interactive perception and articulation.} Robots learn articulation by acting on
objects~\citep{bohg2017interactive,sturm2011articulation,martinmartin2014online,hausman2015active,mo2021where2act},
and latched mechanisms such as lockboxes have been explored with dependency
models~\citep{baum2017lockbox,kulick2015dependency}. These works assume a known set of parts. We add the unseen
hypothesis and show that touch cannot resolve it (\cref{lem:unseen}).

\paragraph{Identification and decisions.} Bayesian model comparison by evidence~\citep{kass1995bayesfactors,bishop2006prml},
simulation-based parameter inference~\citep{ramos2019bayessim,memmel2024asid} and Bayesian experimental
design~\citep{rainforth2024bed} are standard tools. The full problem is a dual-control
problem~\citep{feldbaum1960dual}. We replace it with a one-step, task-directed value of information, a classical
search index~\citep{stone1975search} and a governor~\citep{garone2017refgov}. Objects that may exist but have not been
observed are the subject of open-universe models~\citep{milch2005blog,srivastava2014oupomdp}, and the value of
deliberation under time limits goes back to~\citet{horvitz1991timedependent}.

%=====================================================================
\section{Limitations and conclusion}
\label{sec:conclusion}

\paragraph{Limitations.} Every result is in simulation, and the pictures are renders. The VLM was queried through a
cached file queue, so its latency was simulated. The value of looking again depends strongly on latency
(\cref{tab:looks}), and a timed run against a live model is needed. Every comparison is within one agent. We did not
compare with a vision--language--action policy or a learned adaptive policy trained on the same scenes, so the
results show what declarations add to this agent, not that it beats those families. Several results count tasks,
and the tasks of one scene are not independent; the latency effect is not significant when counted by kitchen. The
Gaussian posterior does not know that inertia, damping and friction are positive. We project samples onto physical
values, and the evidence is computed without that constraint. Through a robot arm the learner loses many tasks to
the arm's inertia and friction, and a robot that stands still cannot follow a wide door. The robot that takes nothing
from the simulator is now ahead of the one told which part each box is, but not significantly, and only after a change
made between two held-out tests.

\paragraph{Conclusion.} A picture shows what a part probably is and how it probably moves. Touch measures how it
really moves. We let a VLM make the first guess and let touch correct it. Because every law, with one coordinate or
two, is linear in its unknown weights, the posterior and the evidence are closed-form, and the floor on each prior
bounds the cost of a wrong claim. Because touch cannot see a part or a motion nobody declared, the robot keeps that
hypothesis open and looks again when a look is worth its time. The VLM helped most where the picture shows what early
touches cannot: a lid's weight, a visible catch, a lift that carries load. It helped least where the picture is
silent. The largest losses came from parts the VLM missed and from the robot's own body, not from wrong physics
guesses. The next steps are a base that moves while it holds a part, so that a wide door can be followed, learning
through the arm, a live VLM with real timing, and a hardware robot.

\paragraph{Code and data.} The simulation code, the scenes and the recorded results are available from the author
on request.

\paragraph{Use of AI tools.} Claude (Anthropic) was used in two ways. In the experiments it was the vision--language
model, as described in \cref{sec:experiments}. In the research it was used, through Claude Code, to help write and run
the simulation and analysis code and to draft and edit the text of this paper. The author directed the work, checked
the results and the text, and takes full responsibility for the content.

\bibliographystyle{plainnat}
\bibliography{\sroot references}

%=====================================================================
\appendix
\section{Posterior and evidence of the regression}
\label{app:nig}

Write $\Phi=\Phi_h$ and $d=d_h$. Given $\sigma^2$, the joint density of the data and the coefficients is
\begin{equation*}
  \N(y;\Phi\beta,\sigma^2I)\,\N(\beta;m_0,\sigma^2V_0)
  =\frac{|V_0|^{-1/2}}{(2\pi\sigma^2)^{(n+d)/2}}\exp\Big(-\frac{Q}{2\sigma^2}\Big),
\end{equation*}
with $Q=(y-\Phi\beta)^{\!\top}(y-\Phi\beta)+(\beta-m_0)^{\!\top}V_0^{-1}(\beta-m_0)$. Expanding $Q$ and collecting the
terms in $\beta$ gives
$\beta^{\!\top}V_n^{-1}\beta-2\beta^{\!\top}V_n^{-1}m_n+y^{\!\top}y+m_0^{\!\top}V_0^{-1}m_0$, with $V_n$ and $m_n$ as in
\cref{eq:post1}. Completing the square,
\begin{equation*}
  Q=(\beta-m_n)^{\!\top}V_n^{-1}(\beta-m_n)+y^{\!\top}y+m_0^{\!\top}V_0^{-1}m_0-m_n^{\!\top}V_n^{-1}m_n
  =(\beta-m_n)^{\!\top}V_n^{-1}(\beta-m_n)+2(b_n-b_0),
\end{equation*}
with $b_n$ as in \cref{eq:post2}. So $\beta\mid\sigma^2,y\sim\N(m_n,\sigma^2V_n)$. Integrating $\beta$ out contributes
$(2\pi\sigma^2)^{d/2}|V_n|^{1/2}$, which leaves
\begin{equation*}
  p(y\mid\sigma^2,h)=(2\pi\sigma^2)^{-n/2}\sqrt{|V_n|/|V_0|}\;\exp\big(-(b_n-b_0)/\sigma^2\big).
\end{equation*}
Multiplying by the prior $\mathrm{IG}(\sigma^2;a_0,b_0)=\frac{b_0^{a_0}}{\Gamma(a_0)}(\sigma^2)^{-a_0-1}e^{-b_0/\sigma^2}$ gives a
function of $\sigma^2$ proportional to $(\sigma^2)^{-a_n-1}e^{-b_n/\sigma^2}$ with $a_n=a_0+n/2$. So
$\sigma^2\mid y\sim\mathrm{IG}(a_n,b_n)$. The integral of $(\sigma^2)^{-a_n-1}e^{-b_n/\sigma^2}$ over $\sigma^2$ is
$\Gamma(a_n)/b_n^{a_n}$, which gives \cref{eq:evidence}:
\begin{equation*}
  p(y\mid h)=(2\pi)^{-n/2}\sqrt{\frac{|V_n|}{|V_0|}}\;\frac{b_0^{a_0}}{\Gamma(a_0)}\;\frac{\Gamma(a_n)}{b_n^{a_n}} .
\end{equation*}

\paragraph{Large $n$.} If $V_n^{-1}=V_0^{-1}+n\Sigma_\phi+o(n)$, then $\log|V_n|=-d\log n-\log|\Sigma_\phi|+o(1)$. If
$b_n=\tfrac n2\hat\sigma_h^2+O(1)$, Stirling's formula $\log\Gamma(a_n)=a_n\log a_n-a_n+O(\log n)$ gives
$-a_n\log b_n+\log\Gamma(a_n)=-\tfrac n2\log\hat\sigma_h^2-\tfrac n2+r_n$. The remainder $r_n$ depends on the structure
only through $O(1)$ terms, because $a_n$ is the same for all structures. Subtracting the log evidence of two
structures fitted to the same data gives \cref{eq:bic}.

\section{Parameters}
\label{app:params}

\begin{table}[h]
\centering
\small
\begin{tabular}{@{}l l l@{}}
\toprule
Symbol & Value & Meaning \\
\midrule
$\Delta t$ & 0.01\,s & control period \\
$\eta$ & 0.05 & floor on block and lock weights \\
$\kappa$ & 0.5 & minimum relative spread of the range prior \\
$\pi_f,\ \pi_j$ & 0.45, 0.04 & prior of \free{} and \jammed{} \\
$\mu$ & 0.40 & prior share of \unseen{} among locks in kitchens (measured) \\
$\rho$ & 0.28 & chance that a look shows a missed releaser in kitchens (measured) \\
$s$ & 0.87 & success with a known releaser and the full budget (measured) \\
$c$ & 0.05 & cost of a look, in success probability (chosen on development) \\
$\gamma$ & 0.2 & gate on $\Prob(\unseen)$ for refusing \\
$Q_{0.01}$ & 1\% & stop quantile used by the governor \\
CVaR level & 25\% & worst share of sampled laws in the planner \\
$T$ & 30\,s & task budget in kitchens (12\,s for valves and objects) \\
\bottomrule
\end{tabular}
\end{table}

\end{document}